\pdfoutput=1

\documentclass[11pt]{article}

\usepackage{amsmath,amsfonts,bm}

\def\Figref#1{Figure~\ref{#1}}

\def\eqref#1{equation~\ref{#1}}

\def\1{\bm{1}}
\newcommand{\train}{\mathcal{D}}

\def\ervc{{\textnormal{c}}}

\def\ervg{{\textnormal{g}}}

\def\ervy{{\textnormal{y}}}
\def\ervz{{\textnormal{z}}}

\def\ermC{{\textnormal{C}}}

\def\ermG{{\textnormal{G}}}

\def\vx{{\bm{x}}}

\def\evv{{v}}
\def\evw{{w}}

\def\mM{{\bm{M}}}

\DeclareMathAlphabet{\mathsfit}{\encodingdefault}{\sfdefault}{m}{sl}
\SetMathAlphabet{\mathsfit}{bold}{\encodingdefault}{\sfdefault}{bx}{n}

\def\gM{{\mathcal{M}}}

\def\sR{{\mathbb{R}}}

\def\emM{{M}}

\usepackage[]{acl}

\usepackage{times}
\usepackage{latexsym}

\usepackage[T1]{fontenc}

\usepackage[utf8]{inputenc}

\usepackage{microtype}

\usepackage{algorithm}
\usepackage{algorithmic}
\usepackage{siunitx}    
\usepackage{relsize}
\usepackage{etoolbox}
\robustify\smaller

\usepackage{newfloat}
\usepackage{listings}
\floatstyle{ruled}
\newfloat{listing}{tb}{lst}{}
\floatname{listing}{Listing}

\usepackage{soul}
\usepackage{amsfonts}
\usepackage{multirow}

\usepackage{dsfont}
\usepackage{commath}

\usepackage{amsmath}
\usepackage{mathtools}

\usepackage{cleveref}

\usepackage{bbold}

\usepackage[utf8]{inputenc}
\usepackage{pgfplots}
\DeclareUnicodeCharacter{2212}{−}
\usepgfplotslibrary{groupplots,dateplot}
\usetikzlibrary{patterns,shapes.arrows}
\pgfplotsset{compat=newest}

\usepackage{placeins}
\usepackage{subcaption}
\usepackage{mwe}
\usepackage{pifont}
\usepackage{xcolor}
\usepackage{paralist}

\usepackage{booktabs}
\usepackage{relsize}
\usepackage{array}

\newcolumntype{V}{>{\smaller}l}

\newcommand{\class}[1]{\textsf{#1}\xspace}

\newcommand{\dataset}[1]{\textsc{#1}\xspace}

\newcommand{\Bios}{\dataset{Bios}}

\newcommand{\method}[1]{\textsc{#1}\xspace}
\newcommand{\Standard}{\method{Standard}}
\newcommand{\INLP}{\method{INLP}}

\newcommand{\Adv}{\method{Adv}}
\newcommand{\DAdv}{\method{DAdv}}

\newcommand{\AAdv}{\method{A-Adv}}

\newcommand{\tpr}{\ensuremath{\text{TPR}}}

\newcommand{\DTO}{\ensuremath{\text{DTO}}\xspace}

\usepackage{bm}

\usepackage{xspace}
\usepackage{adjustbox}

\usepackage{amsthm}
\newtheorem{theorem}{Theorem}[section]
\newtheorem{lemma}[theorem]{Lemma}

\definecolor{applegreen}{rgb}{0.55, 0.71, 0.0}

\newcommand{\newChanges}[1]{#1}

\definecolor{myBlue}{rgb}{0.12156863, 0.46666667, 0.70588235}
\definecolor{myOrange}{rgb}{1., 0.49803922, 0.05490196}
\definecolor{myRed}{rgb}{1,0,0}

\usepackage{float}
\usepackage{mathtools, caption}

\usepackage{xurl}

\usepackage{circuitikz}
\usetikzlibrary{external}
\title{Fair Enough: Standardizing Evaluation and Model Selection for Fairness Research in NLP}

\author{
Xudong Han$^{1,2}$ \qquad  Timothy Baldwin$^{1,2}$ \qquad    Trevor Cohn$^{1}$\\ 
$^{1}$The University of Melbourne \\ 
$^{2}$MBZUAI \\ 
\url{xudongh1@student.unimelb.edu.au}, \url{{tbaldwin,t.cohn}@unimelb.edu.au}
}

\begin{document}
\maketitle
\begin{abstract}

Modern NLP systems exhibit a range of biases, which a growing literature on model debiasing attempts to correct.
However current progress is hampered by a plurality of definitions of bias, means of quantification, and oftentimes vague relation between debiasing algorithms and theoretical measures of bias.
This paper seeks to clarify the current situation and plot a course for meaningful progress in fair learning, with two key contributions:
(1) making clear inter-relations among the current gamut of methods, and their relation to fairness theory; and
(2) addressing the practical problem of model selection, which involves
a trade-off between fairness and accuracy and has led to systemic issues
in fairness research.
Putting them together, we make several recommendations to help shape future work.%
\footnote{Code available at \url{https://github.com/HanXudong/Fair_Enough}}

\end{abstract}

\section{Introduction}

In NLP and machine learning, there has been a surge of interest in fairness\newChanges{~due to the fact that models often learn and amplify biases in the training dataset, leading to a range of harms~\citep{badjatiya2019stereotypical, diaz2018addressing}}.
A central notion is group-wise fairness~\citep{dwork2012fairness, chouldechova2017fair, berk2021fairness}, which is typically measured as the model performance disparities across groups of data that are created by the combinations of protected attributes, such as race and gender.
A broad range of bias evaluation metrics have been introduced in previous studies to capture different types of biases -- such as \emph{demographic parity}~\citep{feldman2015certifying} and \emph{equal opportunity} (EO)~\citep{hardt2016equality} -- and different approaches have been adopted to both measure group disparities within each class, and aggregate over those disparities. Each of these choices implicitly encodes assumptions about the nature of fairness, but little work has been done to spell out what those assumptions are, or guide the selection of evaluation metric from first principles of what constitutes fairness.

\begin{figure}[!t]
    \centering
    \includegraphics[width=\linewidth]{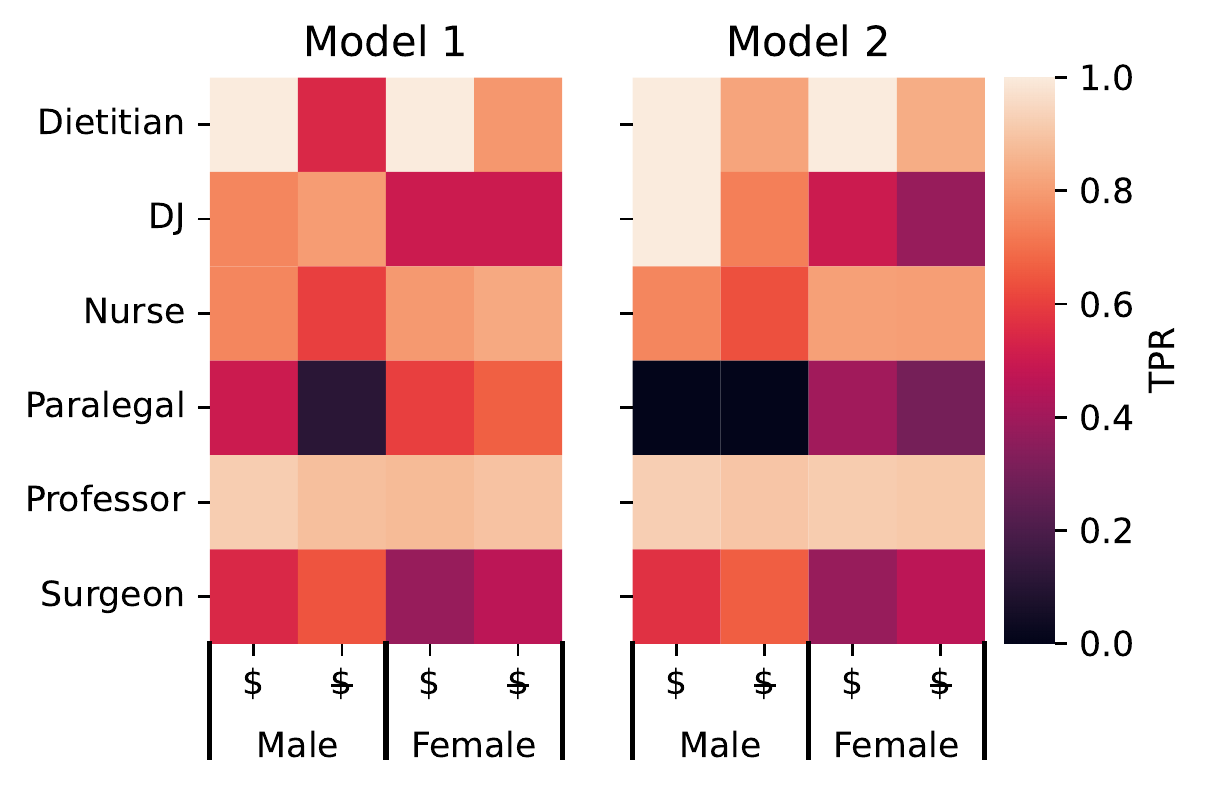}
    \caption{
    True positive rate (TPR) evaluation results over a biography classification dataset broken down by author demographic and selected profession classes. 
    \$ and \st{\$} denote the economic status (wealthy vs.~not, respectively).
    The pattern of results exhibits various biases, however it is difficult to distil this into a single figure of merit, and thus determine which is the better or fairer of the two models.
    }
    \label{fig:tpr_row_metrics}
\end{figure}


As an illustration of this issue, Figure~\ref{fig:tpr_row_metrics} depicts the true positive rate (TPR) values for two  models.\footnote{For further details see Appendix~\ref{sec:introduction_full_results}.}
Given that the EO fairness is satisfied if different groups achieve identical TPR, which model is fairer or ``better'' out of the two?
The answer is far from clear, and in terms of evaluation practice, dictated by a series of choices which implicitly encode different assumptions about what fairness is.



In terms of research practice, these choices have led to a lack of consistency and direct empirical comparability between methods.
Equally concerningly, given that fairness research involves an inherent trade-off between raw model performance and fairness, it has more subtly led to a lack of rigour in terms of how model selection has been carried out, meaning that methods are often deployed in suboptimal ways relative to a particular evaluation methodology.

In this paper, we seek to address these problems. We start by surveying current practices for fairness evaluation aggregation within an integrated framework, and discuss considerations and motivations for using different aggregation approaches. 
To ensure fairness metrics are fully comparable, we present a checklist for reporting fairness evaluation metrics, and also recommendations for aggregation method selections.
We next survey model comparison methods, and demonstrate the issues stemming from using inconsistent model selection criteria. 
To ensure fair comparisons, we further introduce a metric for comparison without model selection, which measures the area under the trade-off curve of each method.

Overall this paper makes two key contributions: (1) we characterise current practices for fairness evaluation and their grounding in theory, proposing a best-practice checklist; and (2) we propose a new method which resolves several issues relating to model selection and comparison.

\begin{table*}
    \centering
    \small
    \begin{tabular}{lllcl}
        \toprule
        {\bf Aggregation} & {\bf Method} & {\bf Description} & {\bf Unit} & {\bf Reference} \\
        \toprule
        \multirow{9}{*}{Group-wise}   & mean gap & $\beta_{\ervc} = \frac{1}{\ermG}\sum_{\ervg}|\emM_{\ervc, \ervg}-\overline{\emM}_{\ervc}|$ & G & \citep{shen2022optimising} \\
        & variance & $\beta_{\ervc} =\frac{1}{G-1}\sum_{\ervg}|\emM_{\ervc, \ervg}-\overline{\emM}_{\ervc}|^{2}$ & G & \citep{lum2022biasing} \\
        & max gap & $\beta_{\ervc} = \text{max}_{\ervg}|\emM_{\ervc, \ervg}-\overline{\emM}_{\ervc}|$ & G & \citep{yang2020fairness} \\
        & min score & $\beta_{\ervc} = \text{min}_{\ervg}\emM_{\ervc, \ervg}$ & S & \citep{NEURIPS2020_07fc15c9} \\
        & min ratio & $\beta_{\ervc} = \text{min}_{\ervg}\frac{\emM_{\ervc, \ervg}}{\overline{\emM}_{\ervc}}$ & R & \citep{zafar2017fairness} \\
        & max difference & $\beta_{\ervc} = \text{max}_{\ervg}\emM_{\ervc, \ervg}-\text{min}_{\ervg}\emM_{\ervc, \ervg}$ & S & \citep{bird2020fairlearn} \\
        & max ratio & $\beta_{\ervc} = \frac{\text{max}_{\ervg}\emM_{\ervc, \ervg}}{\text{min}_{\ervg}\emM_{\ervc, \ervg}}$ & R &  \citep{feldman2015certifying} \\
        & difference threshold ($\gamma$) & $\beta_{\ervc} =\frac{1}{\ermG}\sum_{\ervg}\mathbb{1}_{[0,\gamma]}(|\emM_{\ervc, \ervg}-\overline{\emM}_{\ervc}|)$ & G & \citep{Kearns2019AnES} \\
        & ratio threshold ($\gamma$) & $\beta_{\ervc} = \frac{1}{\ermG}\sum_{\ervg}\mathbb{1}_{[0,\gamma]}(|\frac{\emM_{\ervc, \ervg}}{\overline{\emM}_{\ervc}}-1|)$ & R & \citep{barocas-hardt-narayanan} \\
        \midrule
        \multirow{3}{*}{Class-wise} & binary & $\delta = \sum_{\ervc}\beta_{\ervc}\mathbb{1}_{\{1\}}(c)$ & $\beta$ & \citep{Roh2021FairBatch} \\
        & quadratic mean & $\delta = \sqrt{\frac{1}{\ermC}\sum_{\ervc}\beta_{\ervc}^{2}}$ & $\beta$ & \citep{romanov2019s} \\
        & mean & $\delta = \frac{1}{\ermC}\sum_{\ervc}\beta_{\ervc}$ & $\beta$ & \citep{li-etal-2018-towards} \\
         \bottomrule
    \end{tabular}
    \caption{Summary of different aggregation approaches. Based on the basic unit, group-wise aggregations are additionally categorized into three types: \textbf{S}core ($\emM_{\ervc, \ervg}$), \textbf{G}ap ($|\emM_{\ervc, \ervg}-\overline{\emM}_{\ervc}|$), and \textbf{R}atio ($\frac{\emM_{\ervc, \ervg}}{\overline{\emM}_{\ervc}}$).}
    \label{tab:literature}
\end{table*}

\section{Related Work}
 In terms of bias metrics, there are mainly two lines of work in the literature on NLP fairness: bias in the geometry of text representations (intrinsic bias), and performance disparities across groups in downstream tasks (extrinsic bias), respectively. Based on the hypothesis that measuring and mitigating intrinsic bias will also reduce extrinsic bias, previous work has mainly focused on measuring and mitigating intrinsic bias, such as the Word Embedding Association Test (WEAT)~\citep{caliskan2017semantics}, Sentence Encoder Association Test (SEAT)~\citep{may-etal-2019-measuring}, and Embedding Coherence Test (ECT)~\citep{dev2019attenuating}. However, \citet{goldfarb-tarrant-etal-2021-intrinsic} recently showed that there is no reliable correlation between intrinsic and extrinsic biases, and suggest future work focusing on extrinsic bias measurement (which is the focus of this work). 

As for bias mitigation, debiasing methods for intrinsic and extrinsic bias generally suffer from performance--fairness trade-offs controlled by particular hyperparameters such as the number of principal components used to define the intrinsic bias subspace~\citep{bolukbasi2016man}, and the strength of addition objectives for performance parity across groups~\citep{shen2022optimising}. In measuring performance (perplexity and LM score for sentence embeddings, for example) and fairness simultaneously, the model comparison framework presented in this paper is generalizable for both intrinsic and extrinsic fairness.

\section{Fairness Metrics}
\label{sec:fairness_metric}

In this section, we discuss the considerations involved in fairness evaluation. 
We start with a survey of different methods for aggregating scores, and propose a two-step aggregation framework for fairness evaluation.

\subsection{Formal Notation Preliminaries}
We consider fairness evaluation in a classification scenario. Evaluation is based on a test dataset consisting of $n$ instances $\train=\{(\vx_{i}, \ervy_{i}, \ervz_{i})\}_{i=1}^{n}$, where $\vx_{i}$ is an input vector, $\ervy_{i}\in \{ \ervc \}_{ \ervc = 1 }^{\ermC}$ represents target class label, and $\ervz_{i} \in \{ \ervg \}_{ \ervg = 1 }^{\ermG}$ is the group label, such as gender.\footnote{When considering multiple protected attributes, $\ervz$ can be intersectional identities as shown in~\Cref{fig:tpr_row_metrics}.}

Given a model that has been trained to make predictions w.r.t.~the target label $\hat{\ervy}=f(\vx)$, fairness evaluation metrics generally measure group-wise performance disparities for a particular metric $m(\ervy,\hat{\ervy})$.
For example, positive predictive rate and true positive rate have been employed as the metric for demographic parity~\citep{feldman2015certifying} and equal opportunity~\citep{hardt2016equality}, respectively.

For each group, the results of a metric $m$ are $\ermC$-dimensional vectors, one dimension for each class.
Given $\ermG$ protected groups, the full results are organized as a $\ermC \times \ermG$ matrix, denoted as $\mM$.
For the subset of instances $\train_{\ervc, \ervg}=\{(\vx_{i}, \ervy_{i}, \ervz_{i})|\ervy_{i}=\ervc, \ervz_{i}=\ervg\}_{i=1}^{n}$, we denote the corresponding evaluation results as $\emM_{\ervc, \ervg}$. 
Taking \Figref{fig:tpr_row_metrics} as an example, $\mM$ refers to the heatmap plot, and $\emM_{\ervc, \ervg}$ is the cell in  the $\ervc$-th row and $\ervg$-th column.

Given $\mM$, the question is how exactly to aggregate the result matrix as a single number that measures the degree of fairness. 
We split the aggregation into two steps: (1) group-wise aggregation, which aggregates evaluation results of all groups within a class ($[\emM_{\ervc, 1},\dots,\emM_{\ervc, \ermG}]$) into a single number ($\beta_{\ervc}$); and (2) class-wise aggregation, which aggregates $[\beta_{1},\dots,\beta_{\ermC}]$ scores of all classes into a single number $\delta$.\footnote{
Mathematically, it would be possible to do the class-wise aggregation first, and then the group-wise aggregation. However, aggregating class-wise performances within a particular group essentially measures the long-tail learning problem rather than fairness. 
} 

\subsection{Existing Aggregation Approaches}

Table~\ref{tab:literature} summarizes several aggregation approaches from previous work, which are categorized based on the level of aggregation.

\subsubsection{Basic Unit}

The \textbf{basic unit} refers to the inputs to an aggregation function.

\paragraph{\newChanges{Group-wise}}
Broadly, there are three types of basic units for group-wise aggregation:
\begin{enumerate}
    \item the original \textbf{score} ($\emM_{\ervc, \ervg}$), which maintains the actual performance level under aggregation and larger is better;
    \item the \textbf{gap}, i.e., absolute difference, between the evaluation results of a group and the average performance ($|\emM_{\ervc, \ervg}-\overline{\emM}_{\ervc}|$), where smaller is better and the minimum is 0; and
    \item \textbf{ratio} of the evaluation results of a group to the average ($\frac{\emM_{\ervc, \ervg}}{\overline{\emM}_{\ervc}}$), where closer to 1 is better.
\end{enumerate}

\textbf{Score} describes the actual performance of each group, and is generally used to measure extrema of actual performances. 
For example, the \emph{Rawlsian Max Min} criterion~\citep{rawls2001justice} is satisfied if the utility of the worst-performing group is maximized. 
Related fairness notions are also known as per-group fairness~\citep{hashimoto2018fairness, NEURIPS2020_07fc15c9}.

The other two units, \textbf{gap} and \textbf{ratio}, support the notion of group fairness, and evaluate whether or not $\hat{\ervy}$ is fair w.r.t.~$\ervz$.
Taking EO~\citep{hardt2016equality} as an example, it requires the true positive rate to be independent of $\ervz$.
Formally, for a particular class $\ervc$, the EO criterion is satisfied iff
$$\tpr_{\ervc, \ervg} = \overline{\tpr}_{\ervc}, \forall \ervg \in \{\ervg\}_{\ervg=1}^{\ermG}.$$
As such, it is straightforward to directly measure the absolute difference between $\tpr_{\ervc, \ervg}$ and $\overline{\tpr}_{\ervc}$, 
$$\tpr_{\ervc, \ervg} = \overline{\tpr}_{\ervc} \Leftrightarrow |\tpr_{\ervc, \ervg} - \overline{\tpr}_{\ervc}| = 0,$$
which is essentially the \textbf{gap} unit.

Alternatively, the \textbf{ratio} unit can be used to measure inequality as a percentage:
$$\tpr_{\ervc, \ervg} = \overline{\tpr}_{\ervc} \Leftrightarrow \frac{\tpr_{\ervc, \ervg}}{\overline{\tpr}_{\ervc}} = 1.$$
\textbf{Ratio}-based scores can also be interpreted via a ``$q\%$-rule''~\citep{zafar2017fairness,barocas-hardt-narayanan}, for example, the $80\%$-rule for \emph{disparate impact}~\citep{feldman2015certifying}, which requires that the ratio is no less than 0.8.


The $q\%$-rule can be captured more explicitly by a threshold~\citep{kearns2018preventing,barocas-hardt-narayanan}, which is a relaxation of the equality based on a slack threshold $\epsilon \in \sR^{+}$, 
$|1-\frac{\tpr_{\ervc, \ervg}}{\overline{\tpr}_{\ervc}}| \leq \epsilon$.
Similarly, the threshold can be applied to \textbf{gap}, resulting in $|\tpr{\ervc, \ervg} - \overline{\tpr}_{\ervc}| \leq \epsilon$.


\paragraph{Class-wise}

The next step is class-wise aggregation, taking the group-wise aggregation for each class from above as inputs, $[\beta_{1},\dots,\beta_{\ermC}]$.

\subsection{Generalized Mean Aggregation}
Before discussing each of these aggregation methods, we first introduce the basic concept of the \emph{generalized mean} as a framework for describing aggregation functions, and then make the link between the generalized mean and existing aggregation methods.

Formally, the generalized mean is defined as:
$$\gM_{p}(\evv_{1},\dots,\evv_{n}) = \left(\frac{1}{n}\sum_{i=1}^{n}\evv_{i}^p\right)^{\frac{1}{p}},$$
where $\evv_{i} \in \sR^{+}$ are positive real numbers to be aggregated, and $p$ is the exponent parameter.
A desired property of the generalized mean is its inequality, which states that, 
$$\gM_{p}(\evv_{1},\dots,\evv_{n}) > \gM_{p'}(\evv_{1},\dots,\evv_{n}), \forall p > p'.$$
Essentially, a larger value of $p$ encourage the aggregation to focus more on the larger-valued elements, which can be illustrated with the specific cases shown in Table~\ref{tab:generalized_mean_specific_cases}.

\begin{table}[t]
    \centering
    \small
    \begin{tabular}{cl}
        \toprule
         \bf Power ($p$) & \bf Formulation \\
         \midrule
         $-\infty$ & Minimum: $\min\{\evv_{1}, \dots, \evv_{n}\}$ \\
         $-1$ & Harmonic Mean: $\frac{n}{\sum_{i=1}^{n}\evv_{i}^{-1}}$ \\
         1 & Arithmetic Mean: $\frac{1}{n}\sum_{i=1}^{n}\evv_{i}$ \\
         2 & Quadratic Mean: $\sqrt{\frac{1}{n}\sum_{i=1}^{n}\evv_{i}^{2}}$ \\
         $+\infty$ & Maximum: $\max\{\evv_{1}, \dots, \evv_{n}\}$ \\
         \bottomrule
    \end{tabular}
    \caption{Commonly-used cases of generalized mean aggregation.}
    \label{tab:generalized_mean_specific_cases}
\end{table}

By setting $p = \pm\infty$, generalized mean returns extremum values, including (in Table~\ref{tab:literature}): (1) the maximum value of gap~\citep{yang2020fairness}, difference~\citep{bird2020fairlearn}, and ratio~\citep{feldman2015certifying}; and (2) the minimum value of score~\citep{NEURIPS2020_07fc15c9} and ratio~\citep{zafar2017fairness}.

For other $p$ values, the generalized mean reflects the relative dispersion of its inputs. 
For example, group-wise mean gap aggregation~\citep{shen2022optimising} and class-wise mean aggregation~\citep{li-etal-2018-towards} are both equivalent to $p=1$. 
Class-wise quadratic mean aggregation~\citep{romanov2019s} is essentially $p=2$, which focuses more on those classes with higher bias.
Similarly, group-wise variance aggregation~\citep{lum2022biasing} is proportional to the $p=2$ setting, implying that groups with larger gaps will influence results more.

The additional advantage of using generalized mean aggregation is that comparison across arbitrary $p$ values can be easily stated. For example, group-wise aggregation is the $p=-5$ generalized mean with respect to the \textbf{score} units in a toxicity classification competition,\footnote{Jigsaw Unintended Bias in Toxicity Classification: \url{https://www.kaggle.com/competitions/jigsaw-unintended-bias-in-toxicity-classification/}}  meaning that evaluation focuses more on groups with lower performance.

\paragraph{Other Aggregation Methods:}
Although generalized mean aggregation is a powerful tool for describing and interpreting the aggregation process, there are other ways that need further discussion.
Previous work has also considered assigning different weights under aggregation, for instance, \citet{kearns2018preventing} assign larger weights to groups with larger populations. 
Such aggregations can be implemented as the weighted generalized mean:
$\gM_{p, \evw}(\evv) = \left(\frac{1}{n}\sum_{i=1}^{n}\evw_{i}\evv_{i}^p\right)^{\frac{1}{p}}$, where $\evw$ is the weight vector, and $\sum_{i=1}^{n}\evw_i=1$.

An example of the weighted generalized mean for class-wise aggregation is binary aggregation that only considers the positive class in a binary classification setting~\citep{hardt2016equality, zafar2017fairness,kearns2018preventing,zhao2019conditional,NEURIPS2020_07fc15c9,han2021diverse,lum2022biasing}.
The positive class is often treated as the ``advantaged'' outcome, so the analysis focuses solely on the positive class.
Moreover, the one-versus-all trick is not necessary for the binary setting, and natural derivations of the confusion matrix can be used to refer to a particular class, e.g., TPR for the positive class and TNR for the negative class.

\begin{figure*}[!ht]
    \centering
    \includegraphics[width=0.95\textwidth]{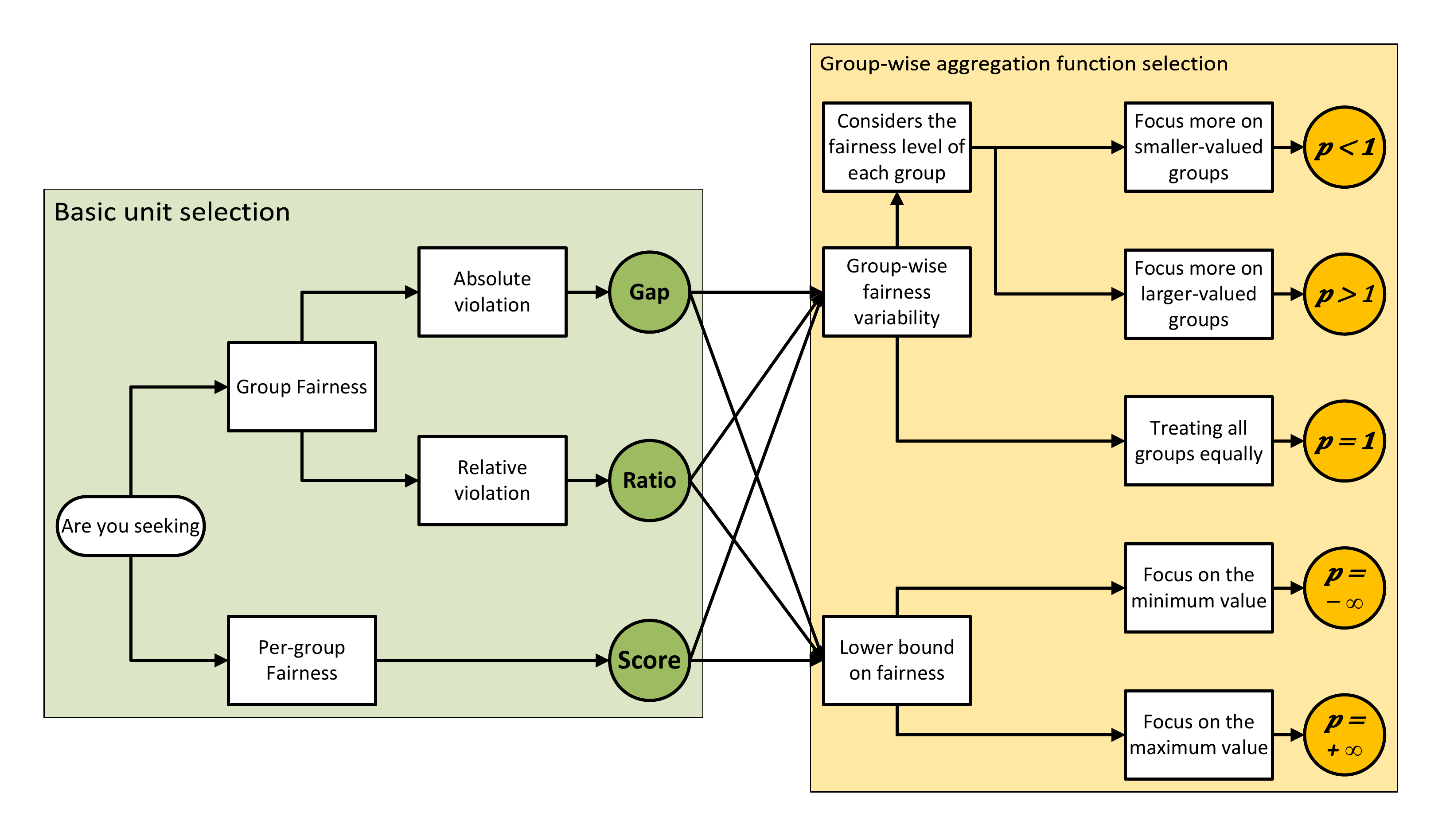}
    \caption{Decision path for exponent parameter selection for generalized mean aggregation. 
    }
    \label{fig:Decision_path}
\end{figure*}

\subsection{Recommendations}

We are now in a position to be able to provide recommendations for fairness evaluation. 

Following the work of \citet{dodge-etal-2019-show}, we provide a checklist for fairness evaluation metric aggregation:
\begin{itemize}
    \item[$\square$] Statistics of the dataset $\train$, e.g., the probability table of the joint distribution of $\ervy$ and $\ervz$, and the size of each partition.
    \item[$\square$] The evaluation metric $m$ (e.g., TPR for \emph{EO} fairness).
    \item[$\square$] The basic unit of group-wise aggregation, including \textbf{score}, \textbf{gap}, and \textbf{ratio}, or other possible measures.
    \item[$\square$] The aggregation function for \textbf{group-wise} aggregation, and the corresponding motivation.
    \item[$\square$] The aggregation function for \textbf{class-wise} aggregation, and the corresponding motivation.
\end{itemize}

Although the particular choice of evaluation dataset $\train$ and evaluation metric $m$ are critical to the overall evaluation, they are not the main focus of this paper. 
Rather, we provide guidance based on the selection of basic unit, and methods for group- and class-wise aggregation, as detailed in  \Figref{fig:Decision_path}.

\paragraph{Basic Unit Selection:}
The circles in  \Figref{fig:Decision_path} annotated as \textbf{Score}, \textbf{Ratio}, and \textbf{Gap} are the decision points for basic unit selection.

If per-group fairness is the primary criterion (e.g., Rawlsian Max-Min fairness~\citep{rawls2001justice}), using \textbf{score} is the best practice, which maintains the original values under aggregation.
On the other hand, if inter-group fairness is critical, \textbf{gap} and \textbf{ratio} are more appropriate choices. \textbf{Gap} reflects disparities in the same scale as the per-group scores, and is easy to visualize (e.g.\ as differences in height between clustered bars). However, if one wished to measure disparities in relative terms, e.g., the $q\%$-rule~\citep{feldman2015certifying}, \textbf{ratio} is a better choice than \textbf{gap}.

\paragraph{Group-wise Aggregation Function Selection:}
The selection of group-wise aggregation functions is shown as the exponent parameters of the \emph{generalized mean} aggregation.

Measuring extrema is similar to the notion of per-group fairness, and encourages improvements in the worst-performing groups.
For basic units where smaller is fairer, e.g., \textbf{gap}, aggregation generally focuses on the maximum~\citep{yang2020fairness}, i.e., $p=+\infty$.
For units like \textbf{score} ($p=-\infty$), on the other hand, the minimum value should be measured, as a lower bound.

Besides extrema, it is also reasonable to measure fairness variability across groups.
A typical choice is taking the arithmetic mean (i.e., $p=1$) across all groups, which implicitly assigns equal importance to each individual group.
Similar to the signs in extremum aggregations, the value of $p$ in variability aggregations should be selected based on the type of basic unit, to focus more on worse-performing groups.
Taking the \textbf{gap} unit as an example, the quadratic mean ($p=2$) is influenced more by larger gaps than the arithmetic mean. 
Moreover, quadratic mean aggregation based on \textbf{gap} is essentially the standard deviation of \textbf{scores}, and can be used to reconstruct variance aggregation~\citep{lum2022biasing}.

\paragraph{Class-wise Aggregation Function Selection:}
Although our focus is on the fairness evaluation metric, class-wise aggregation is almost identical to aggregation methods for general utility metrics.
Binary aggregation for fairness is the same as utility metrics, while mean aggregation~\citep{li-etal-2018-towards,wang2019balanced} for fairness evaluation is equivalent to ``macro''-averaging in general evaluation.

\section{Model Comparison}
\label{sec:model_comparison}

This section focuses on comparison of debiasing methods when considering utility and fairness simultaneously.
We first introduce the performance--fairness trade-off curve (PFC) for debiasing methods, and then discuss the limitations of existing comparison frameworks.
Finally, we propose a new metric, namely the area under the curve (AUC) w.r.t.\ PFC, which integrates existing approaches and reflects the overall goodness of a method. 

\begin{figure}[!t]
    \centering
     \begin{subfigure}[b]{0.45\textwidth}
         \centering
         \includegraphics[width=\linewidth]{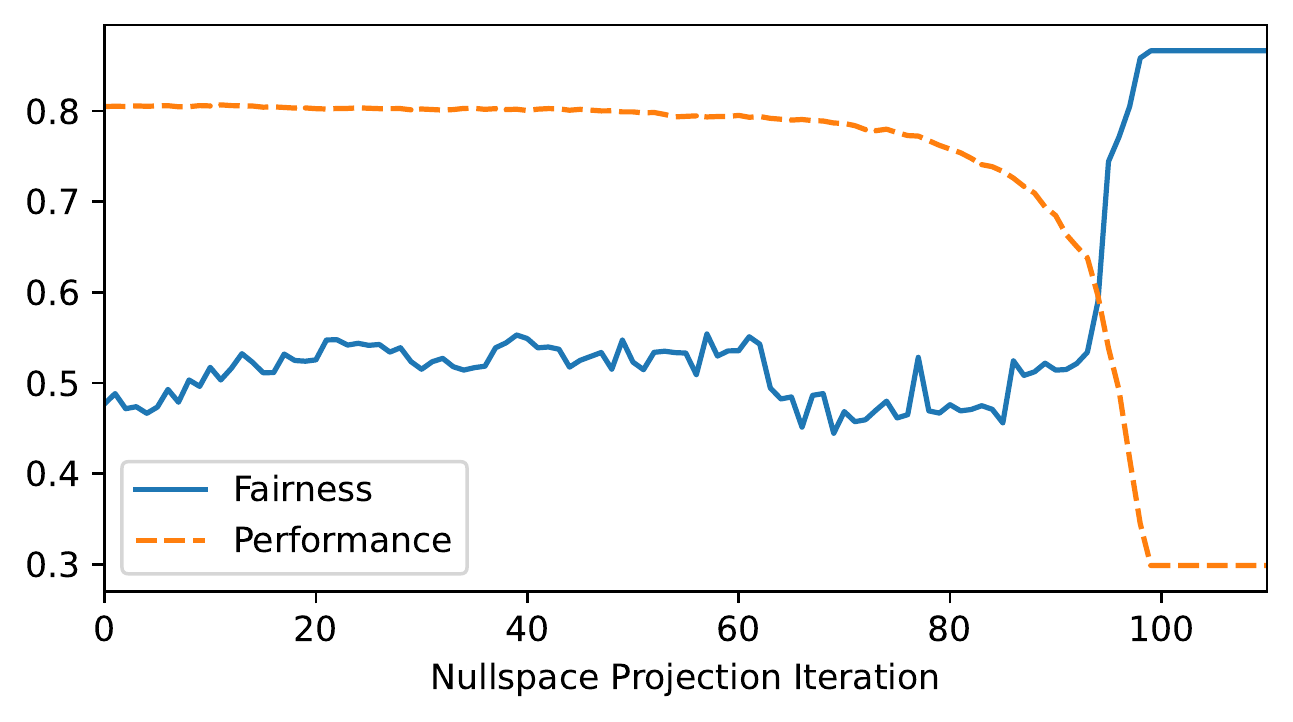}
         \caption{Tuning \INLP trade-off hyperparameter}
         \label{fig:Bios_both_INLP_Iteration}
    \end{subfigure}
    \\ \hfill \hfill 
    \begin{subfigure}[b]{0.45\textwidth}
        \centering
        \includegraphics[width=\linewidth]{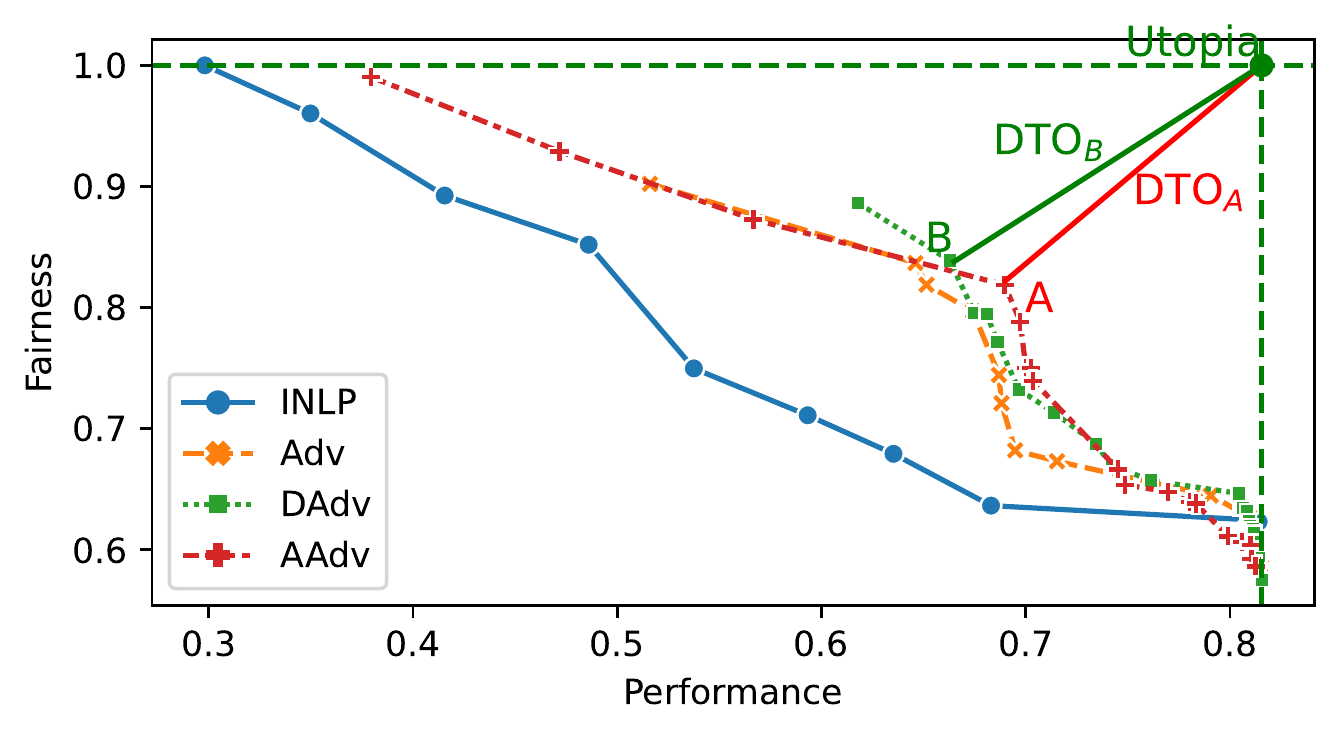}
        \caption{Performance--fairness trade-offs}
        \label{fig:Bios_both_sample_tradeoffs}
    \end{subfigure}
    \hfill
    \caption{\Figref{fig:Bios_both_INLP_Iteration} shows an example of performance and fairness with respect to different iterations of the nullspace projection in \INLP.
    \Figref{fig:Bios_both_sample_tradeoffs} presents the Pareto frontiers in the performance--fairness trade-offs of four debiasing methods in recent work. 
    \Figref{fig:Bios_both_sample_tradeoffs} also provides an illustration of \DTO. 
    The green dashed vertical and horizontal lines denote the best performance and fairness, respectively, and their intersection point is the \emph{Utopia} point. The length of a line, e.g., the red line from A to the Utopia point, is the \DTO for the corresponding candidate model.}
    \label{fig:INLP_tune}
\end{figure}


\subsection{Performance and Fairness Metrics}
\label{sec:performance_and_fairness_metrics}
As discussed in \Cref{sec:fairness_metric}, there are many options to measure performance and fairness.
This paper is generalizable to all different metrics, but for illustration purposes, we follow~\citet{ravfogel-etal-2020-null,subramanian-etal-2021-evaluating} and~\citet{han2022fairlib} in measuring the overall accuracy and equal opportunity fairness.

Specifically, equal opportunity fairness measures TPR disparities across groups, such as the situation depicted in Figure~\ref{fig:tpr_row_metrics}. 
We use the TPR \textbf{gap} across subgroups to capture absolute disparities. 
For group-wise aggregation, we treat all groups equally in computing the unweighted sum of $gap$ scores ($\propto p=1$).
In the last step, class-wise aggregation, we focus more on less fair classes by using root mean square aggregation ($p=2$).

\subsection{Performance--Fairness Trade-off}
It has been observed in previous work that a performance--fairness trade-off exists in bias mitigation~\citep{li-etal-2018-towards,wang2019balanced,ravfogel-etal-2020-null,han2022towards,shen2022optimising}.

Typically, debiasing methods involve a trade-off hyperparameter to control the extent to which the model sacrifices performance for fairness.
Examples of such trade-off hyperparameters include: (1) interpolation between the target and vanilla data distribution for pre-processing approaches~\citep{wang2019balanced,han2021balancing}; (2) the strength of additional loss terms for loss manipulation methods~\citep{zhao2019conditional,NEURIPS2020_07fc15c9, han2021diverse,shen2022does}; (3) the target level of fairness in constrained optimization~\citep{kearns2018preventing,subramanian-etal-2021-evaluating}; and (4) the number of debiasing iterations for post-hoc bias mitigation methods~\citep{ravfogel-etal-2020-null}.

Taking \INLP~\citep{ravfogel-etal-2020-null} as an example, which debiases by iteratively projecting the text embeddings to the nullspace of the protected attributes, \Figref{fig:Bios_both_INLP_Iteration} shows performance and fairness with respect to the number of nullspace projection iterations.\footnote{Without loss of generality, we assume that for both fairness and performance, larger is better.}
It is clear that more iterations lead to better fairness at the cost of performance.

Instead of looking at performance/fairness for different trade-off hyperparameter values, it is more meaningful to focus on the Pareto frontiers in trade-off plots (\Figref{fig:Bios_both_sample_tradeoffs}), where each point 
corresponds to a particular value of the trade-off hyperparameter in \Figref{fig:Bios_both_INLP_Iteration}.
The frontiers represent the best fairness that can be achieved at different performance levels, and vice versa.

One limitation of a trade-off plot is that it is hard to make quantitative conclusions based on the plot itself, and we cannot conclude that one method is better than another if there exists any intersection of their trade-off curves.
As shown in \Figref{fig:Bios_both_sample_tradeoffs}, in addition to \INLP, we also include the trade-off curves for three recent adversarial debiasing variants: \Adv~\citep{li-etal-2018-towards}, \DAdv~\citep{han2021diverse}, and \AAdv~\citep{han2022towards}.
Although \AAdv is better than the other methods under most conditions, there exist intersections between their trade-off curves.
As such, we can only state that \AAdv is better than other methods within particular ranges, which is insufficient for making a precise comparison, especially when comparing multiple debiasing methods (as demonstrated in \Figref{fig:Bios_both_sample_tradeoffs}).

\subsection{Model Selection}
\label{sec:Model_Selection}

In order to conduct quantitative comparisons across different debiasing methods, current practice is to select a particular point on the frontier for each method, and then compare both the performance and fairness of the selected points.

One problem associated with model selection is that typically, no single method simultaneously achieves the best performance and fairness. 
For example, as shown in \Figref{fig:Bios_both_sample_tradeoffs}, if points \textbf{A} and \textbf{B} were the selected models for \AAdv and \DAdv, respectively, \textbf{A} would represent better performance and \textbf{B} better fairness. 
As such, although we have actual numbers for quantitative comparison, it is still hard to conclude which method is best.

\paragraph{Distance to the Optimal:}
To address this problem, we propose to measure the \textbf{D}istance \textbf{T}o the \textbf{O}ptimal point (``DTO'') to quantify the performance--fairness trade-off~\citep{salukvadze1971concerning, marler2004survey, han2021balancing}.
A model is said to outperform others if it achieves a smaller \DTO, i.e.\ the distance to the optimal (Utopia) point (the point at which performance and fairness are the maximum possible values) is minimized.
Figure~\ref{fig:Bios_both_sample_tradeoffs} illustrates the calculation of \DTO for \textbf{A} and \textbf{B}, where the optimal point is the top-right corner\footnote{The location of the utopia point \newChanges{and the scale of metrics are} discussed in \Cref{sec:AUC_extension}.} and \DTO is measured by the normalized Euclidean distance (the length of the green and red lines) to the optimal point.

A notable advantage of \DTO is that a Pareto improvement implies a smaller value of \DTO.
Therefore, DTO can be seen as relaxation of Pareto improvements, and the smallest DTO must be achieved by a point on the Pareto frontier.
A key limitation of \DTO is that it quantifies the trade-off of a single model rather than the full frontier, presupposing some means of model selection.
This has been somewhat arbitrary in prior work, which is the problem we now seek to address.

\begin{figure}[!t]
    \centering
    \includegraphics[width=\linewidth]{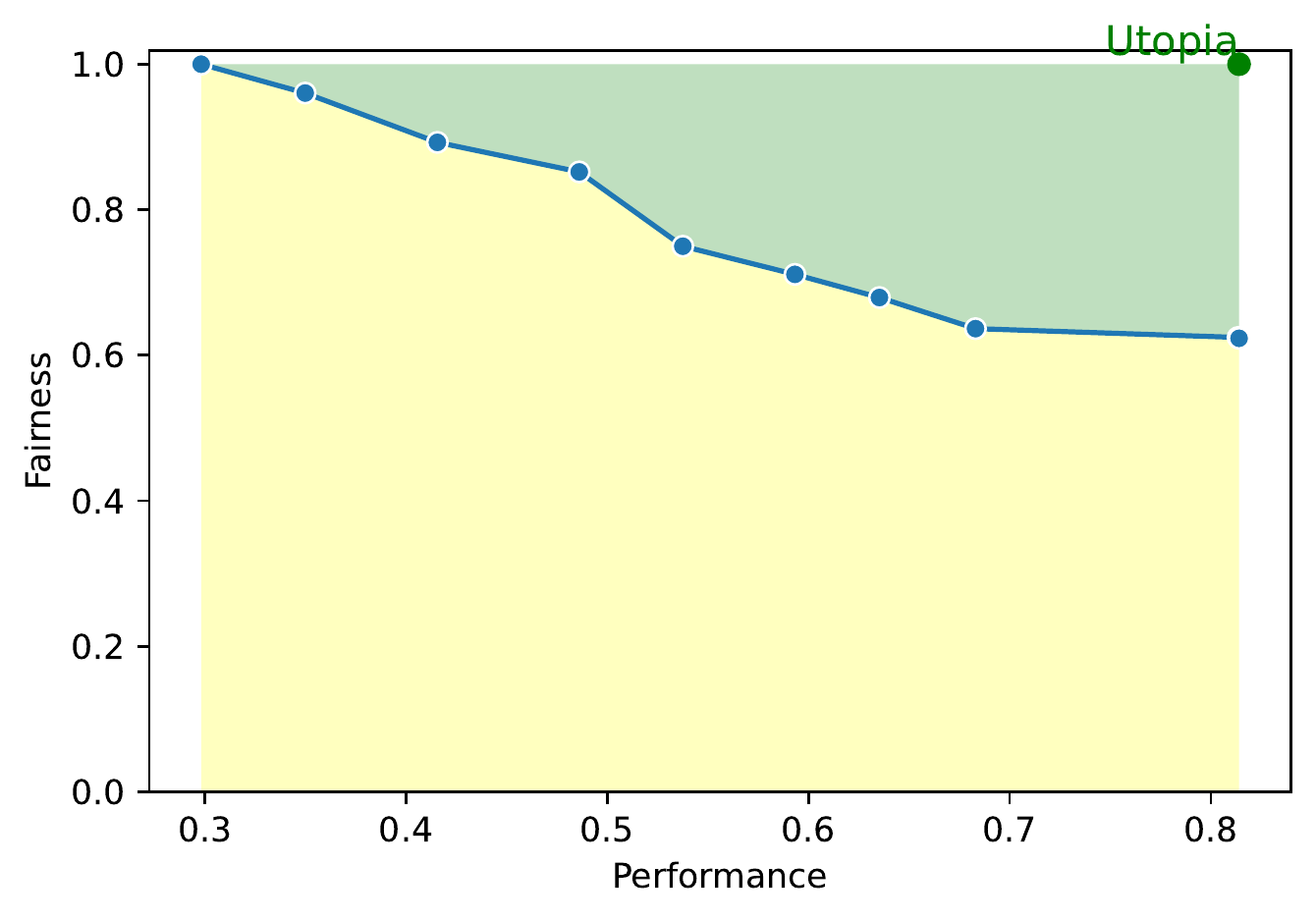}
    \caption{AUC of the performance--fairness trade-off curve. Taking the trade-off curve of \INLP as an example, the yellow shaded area refers to the AUC-PFC, and the green shaded area refers to the integration of DTO in polar coordinates. }
    \label{fig:auc_pfc}
\end{figure}

\paragraph{Selection Criteria:}
Similar to the aggregation of fairness metrics, model selection should be done in a domain-specific manner. 
Previous work has used different criteria for model selection, including: (1) minimum loss~\citep{hashimoto2018fairness,li-etal-2018-towards}; (2) maximum utility~\citep{NEURIPS2020_07fc15c9}, e.g., based on accuracy or F-measure; (3) manual selection based on visual inspection of the trade-off curve~\citep{elazar2018adversarial, ravfogel-etal-2020-null}; (4) constrained selection~\citep{han2021diverse,subramanian-etal-2021-evaluating}, by selecting the best fairness constrained to a particular level of performance, and vice versa; and (5) minimising \DTO~\citep{han2022towards,shen2022optimising}.

\begin{figure*}[ht]
     \centering
     \begin{subfigure}[b]{0.3\textwidth}
         \centering
         \includegraphics[width=\textwidth]{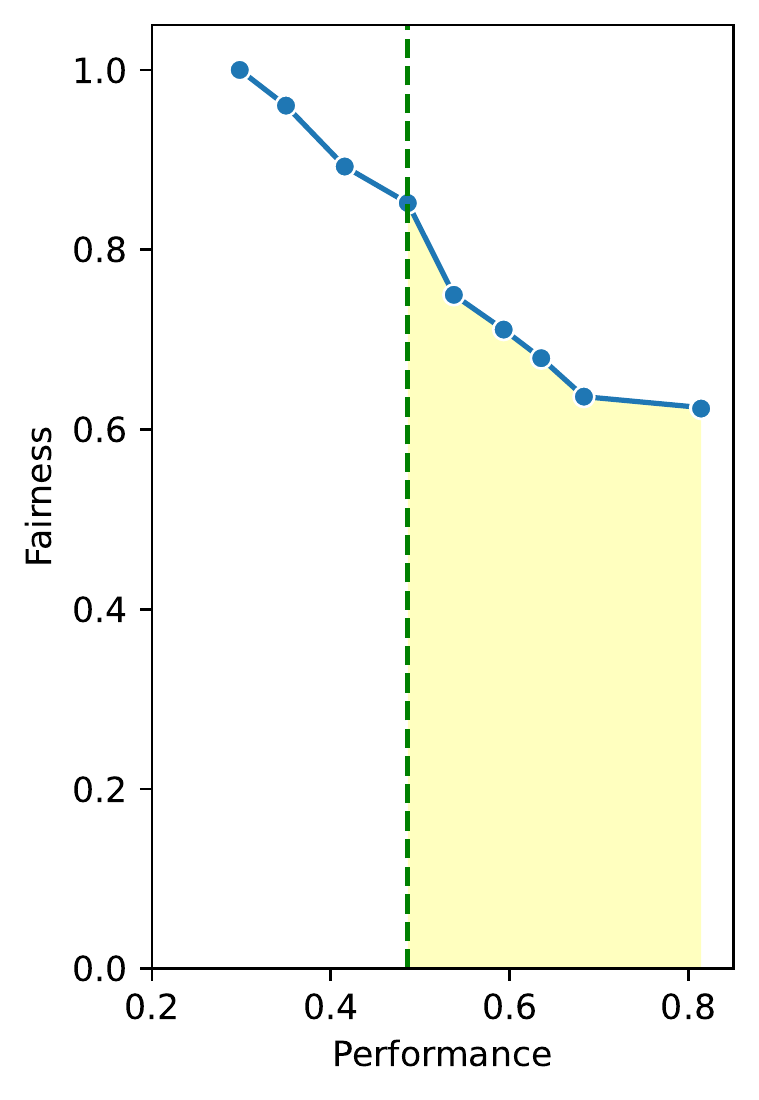}
         \caption{Accuracy is better than $0.49$.}
         \label{fig:AUC_performance_constrained}
     \end{subfigure}
     \hfill
     \begin{subfigure}[b]{0.3\textwidth}
         \centering
         \includegraphics[width=\textwidth]{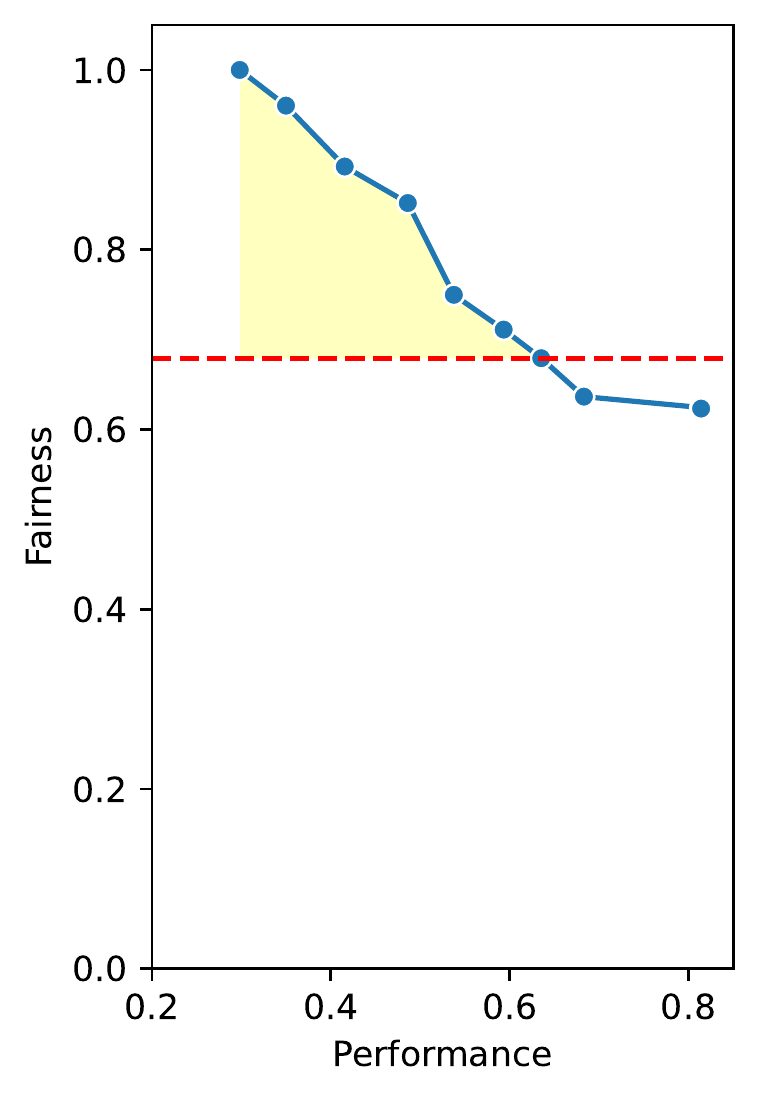}
         \caption{Fairness is better than $0.68$.}
         \label{fig:AUC_fairness_constrained}
     \end{subfigure}
     \hfill
     \begin{subfigure}[b]{0.3\textwidth}
         \centering
         \includegraphics[width=\textwidth]{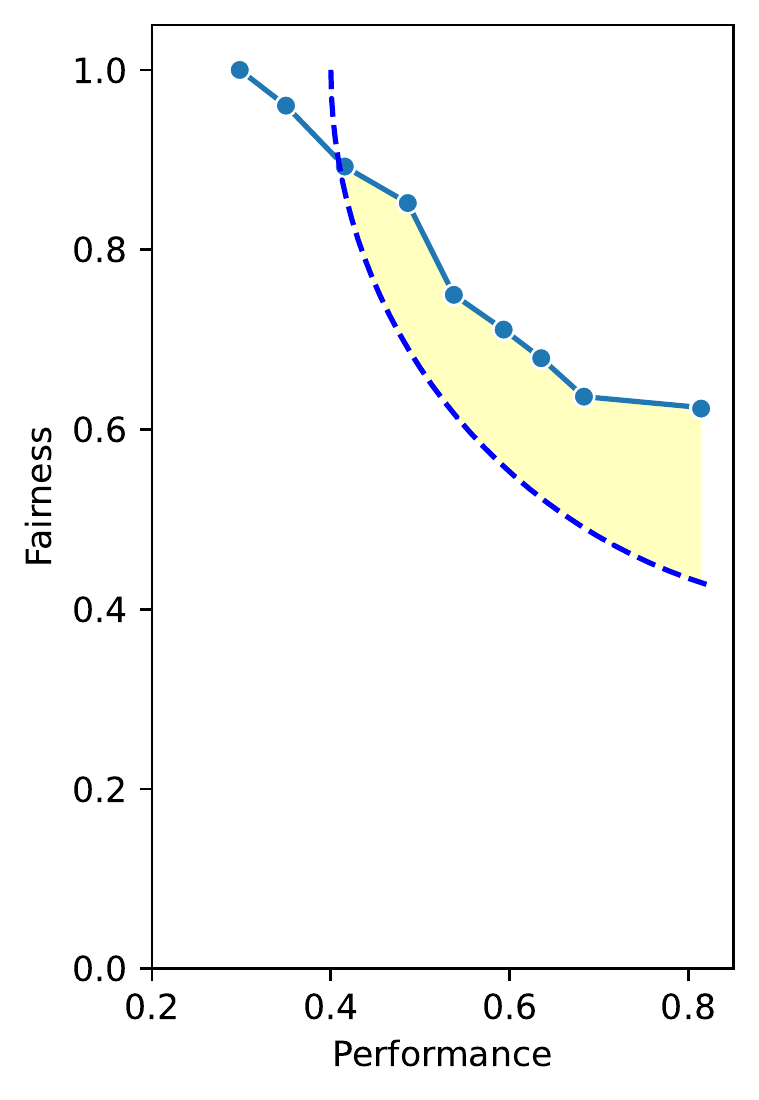}
         \caption{DTO is better than $0.60$.}
         \label{fig:AUC_dto_constrained}
     \end{subfigure}
        \caption{Yellow shaded area denote the partial AUC-PFC score computed in the region where a particular condition applied.}
        \label{fig:partial_auc}
\end{figure*}

Selection based on minimum loss and maximum utility is identical to classic model selection, and does not consider fairness explicitly.
The other three types of criteria are based on trade-offs, differentiated by the method for aggregating fairness and performance.

Such inconsistency in model selection makes it very hard to rigorously compare methods.
The question we want to address is: \emph{how can we quantitatively compare methods without model selection?} 

\subsection{AUC-PFC}
\label{sec:AUC-PFC}

Recall that \DTO is a metric for measuring the goodness of the trade-off of a particular model, and model selection is a process for selecting a particular frontier model from the Pareto curve.
To address the problem associated with model selection, we propose to integrate \DTO over the whole performance--fairness curve (PFC).
Specifically, we integrate \DTO in a polar coordinate system, where the reference point (pole) is the optimal point.
Given that the DTO of a point on the trade-off curve is also the distance from the pole in the polar coordinate system, the trade-off curve can be treated as a function that maps angular coordinates to DTO.
For example, as shown in \Figref{fig:auc_pfc}, the green area denotes the region enclosed by the performance--fairness trade-off curve of \INLP  and the utopia point with fairness = $1.00$ and performance = $0.82$.\footnote{In the interests of consistent comparison, the Utopia point is typically $(1,1)$, as in Table~\ref{table:part_selected_results}. In practice, this does not affect the calculation of AUC-PFC, as we discuss in Appendix~\ref{sec:AUC_extension}.}

Alternatively, we can interpret the proposed metric from the performance--fairness perspective, in calculating the area under the Pareto curve, and subtracting this from the area under the optimal Pareto curve defined by the optimal point.

\newChanges{The magnitude of AUC-PFC differs from a single metric; for example, a $0.0001$ improvement in the AUC-PFC score is equivalent to a 1 percentage point (pp) boost in both performance and fairness ($0.01 \times 0.01$).}

\paragraph{Partial AUC-PFC}
In practice, worse performance or fairness can be unacceptable, for example, one may want to prioritize fairness in particular applications.
To address this problem, we present the Partial AUC-PFC score to focus on a specific region of the PFC Curve, where the AUC-PFC score is computed w.r.t.\ specific acceptable levels of performance and fairness.

\Cref{fig:AUC_performance_constrained} shows an example of partial AUC-PFC, where the region can be considered if corresponding accuracy is better than $0.49$. 
Similarly, Figures~\ref{fig:AUC_fairness_constrained} and~\ref{fig:AUC_dto_constrained} show partial AUC-PFC scores with respect to particular fairness and \DTO constraints. 

With the partial AUC-PFC metric, one can explicitly compare different methods with a single number, and w.r.t.\ particular values of performance and fairness.

\begin{table*}[!t]
\renewrobustcmd{\bfseries}{\fontseries{b}\selectfont}
\centering
\small
\sisetup{
round-mode = places,
}%
\begin{tabular}{
l
S[table-format=2.1, round-precision = 1]
S[table-format=2.1, round-precision = 1]
S[table-format=2.1, round-precision = 1]
S[table-format=2.1, round-precision = 1]
S[table-format=2.1, round-precision = 1]
S[table-format=2.1, round-precision = 1]
S[table-format=2.1, round-precision = 1]
S[table-format=2.1, round-precision = 1]
}

\toprule
 & \multicolumn{7}{c}{Selection Criteria} & \\
\cmidrule(lr){2-8}
\bf Method  & \multicolumn{1}{c}{\bf \DTO} & \multicolumn{1}{c}{\bf P} & \multicolumn{1}{c}{\bf P@F$+$5\%} & \multicolumn{1}{c}{\bf P@F$+$10\%} & \multicolumn{1}{c}{\bf F} & \multicolumn{1}{c}{\bf F@P$-$5\%} & \multicolumn{1}{c}{\bf F@P$-$10\%} & AUC $\uparrow$\\ 
\midrule
  \INLP~\citep{ravfogel-etal-2020-null} & 41.931134 & \bfseries{41.9} & 52.580521 & 52.580521 & 70.187785 & 41.931134 & 41.931134 & 39.82569166553376\\
   \Adv~\citep{li-etal-2018-towards} & 38.977707 & 44.560375 & 43.344399 & 41.841526 & 49.377388 & 41.224841 & \bfseries{41.2} & 43.572503606466517\\
  \DAdv~\citep{han2021diverse} & 37.943508 & 44.744975 & 41.011131 & 40.454108 & \bfseries{39.9} & \bfseries{40.4} & 41.911101 & \bfseries{44.5}\\
  \AAdv~\citep{han2022towards} & \bfseries{36.9} & 45.413214 & \bfseries{39.5} & \bfseries{39.0} & 62.060511 & 43.768721 & 42.796501 & 44.00661978816351\\
\bottomrule
\end{tabular} 
\caption{\DTO scores of selected models over the \Bios dataset (smaller is better), based on the distances from mean performance and fairness to (1,1) over the test set. Models are selected based on the criterion listed for each column over the development set. The final column is the AUC, which does not involve model selection. \textbf{Bold} = the best score per column. See Appendix~\ref{sec:full_comparison} for the full results.}
\label{table:part_selected_results}
\end{table*}

\subsection{Case Studies}
\label{sec:case_studies}

\paragraph{Experimental Details:}
We conduct experiments over the \Bios dataset~\citep{de2019bias} which was augmented with economic status by~\citet{subramanian-etal-2021-evaluating}, resulting in 28 professions as the target label and 4 intersectional demographic groups.\footnote{Performance and fairness metrics have been introduced in~\Cref{sec:performance_and_fairness_metrics}.}
We use public implementations for all models in our experiments, primarily in \emph{fairlib}~\citep{han2022fairlib}. 

\paragraph{Results:}
In \Cref{table:part_selected_results}, we investigate 7 different selection criteria and report the \DTO score over the test set. 
Specifically, we conduct model selection over the development set based on: (1) minimum \DTO; (2) maximum performance (\textbf{P}); (3) maximum performance within a fairness threshold of 5\% improvement (\textbf{P@F$+$5\%}); (4) maximum performance within a fairness threshold of 10\% improvement (\textbf{P@F$+$10\%}); (5) maximum fairness (\textbf{F}); (6) maximum fairness within a performance trade-off threshold of 5\% (\textbf{F@P$-$5\%}); and (7) maximum fairness within a performance trade-off threshold of 10\% (\textbf{F@P$-$10\%}).
Criteria (3), (4), (6), and (7) are constrained selections as discussed in \Cref{sec:Model_Selection}, where we select the model with the highest fairness/performance within 5/10\% of performance trade-off/fairness improvement relative to the \Standard model.
Taking F@P$-$10\% as an example, the model with the highest fairness is selected within 10\% performance trade-off over the vanilla model performance (i.e., with performance greater than 72\% ($82\%-10\%$)).
Similarly, P@F$+$5\% selects the model with highest performance subject to at least 5\% fairness improvement over the  vanilla model ($63\%=58\%+5\%$).


It can be seen that each of the four methods is the best for at least one selection criteria, as a stark illustration of our claim about model selection criteria biasing any possible conclusions about which method is best. 
For example, \INLP and \Adv are the best methods with respect to selection criteria \textbf{P} and \textbf{F@P$-$10\%}, respectively.

On the contrary, our proposed AUC-PFC score (in the final column of \Cref{table:part_selected_results}) is unaffected by model selection, and reflects the overall trade-off of a method. 
Consistent with the trend in \Cref{fig:Bios_both_sample_tradeoffs}, AUC scores in \Cref{table:part_selected_results} are smaller for worse-performing methods, e.g., \INLP, and larger for better-performing methods such as \AAdv and \DAdv.
Moreover, as the trade-off curves for \AAdv and \DAdv overlap one another (see \Cref{fig:Bios_both_sample_tradeoffs}), it is hard to pick a winner visually, let alone make quantitative comparisons.
By using the AUC-PFC metric, we can conclude that overall, \DAdv is slightly better than \AAdv over this dataset.

\paragraph{Discussion}
The current DTO calculation assumes that users have no preference for performance over fairness or vice versa, where in practice it is possible that the choice of the fairness metric could be influenced by task-specific goals or the relative importance of fairness. Such problems have been widely studied in the literature on multi-objective learning, and a typical line of work is weighted generalized mean, which incorporates additional weight parameters in the generalized mean framework to reflect the importance or preference of each objective. 

\section{Conclusion}
We have discussed the current practice in evaluation, model selection, and method comparison in the fairness literature, and shown how current practice in experimental fairness lacks rigour and consistency. 
We made recommendations for selecting a fairness evaluation metric, and introduced a new metric for measuring the overall performance--fairness trade-off of a method.

\section*{Acknowledgements}
We thank Lea Frermann, Aili Shen, and Shivashankar Subramanian for their discussions and inputs.
We thank the anonymous reviewers for their helpful feedback and suggestions. 
This work was funded by the Australian Research Council, Discovery grant DP200102519.

\section*{Limitations}
This paper focuses on the notion of group fairness, under the assumption that each individual belongs to a particular demographic group. One
limitation of methods in this space is that the demographic attributes must be observed (for the development and test data, at least)  in order to evaluate fairness.


We only investigate the proposed evaluation aggregation framework in a classification setting. However, our framework is naturally generalizable to other tasks with discrete outcomes, such as generation and sequential tagging. Moreover, in terms of continuous labels, such as regression, one can skip the class-wise aggregation.

\section*{Ethical Considerations}
This work focuses on current practice in fairness evaluation and method comparison. \newChanges{Our proposed ``checklist'' recommendations are specific to the fairness literature and complement existing frameworks, to encourage future research to think carefully about harms and what type of fairness is appropriate.}

Demographics are assumed to be available only for evaluation purposes and are not used for model training or inference.
We only use attributes that the user has self-identified in our experiments. 
All data and models in this study are publicly available and used under strict ethical guidelines.

\bibliography{custom}
\bibliographystyle{acl_natbib}

\clearpage
\appendix

\section{Full Disaggregated Results}
\label{sec:introduction_full_results}

\begin{figure}[!t]
    \centering
    \includegraphics[width=\linewidth]{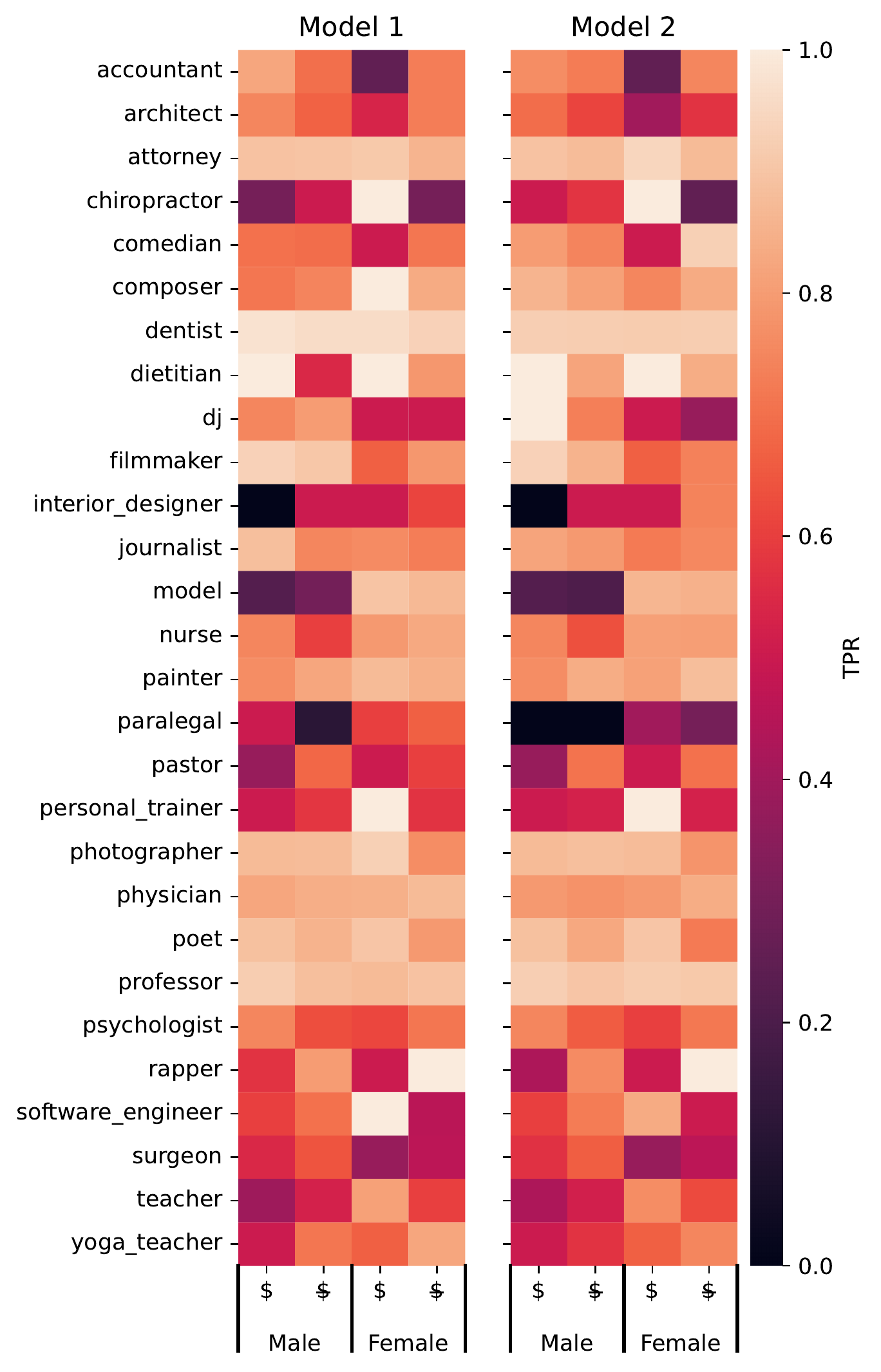}
    \caption{True positive rate (TPR) evaluation results over a biography classification dataset broken down by author demographics and profession classes. 
      \$ and \st{\$} denote the economic status (high vs.\ low, respectively). }
    \label{fig:annotated_tpr_row_metrics_v}
\end{figure}

Figure~\ref{fig:annotated_tpr_row_metrics_v} depicts the full TPR scores for two real-world models over a profession classification dataset, stratified across 4 protected attributes (\class{male} vs.\ \class{female} and \class{high} vs.\ \class{low} economic status) from \Cref{fig:tpr_row_metrics}. Specifically, the two models are both trained naively without debiasing. They share the same hyperparameter settings and random seed, except that Model 1 and Model 2 are the 9th and 5th epochs, respectively.
For professions such as \class{Professor}, there is little discernible difference either between the two models or across different combinations of protected attributes.
For \class{DJ}, on the other hand, Model 1 appears to be reasonably fair w.r.t.\ economic status but biased for binary gender, whereas Model 2 is biased across both protected attributes but appears to have the higher overall TPR.
Finally, with \class{Paralegal}, Model 2 appears to be fairer w.r.t.\ both economic status and binary gender but perform substantially worse than the more biased Model 1 in terms of the individual TPR scores for every combination of protected attributes.
So it is hard to tell which model is fairer or ``better'' out of the two, without aggregation.

\begin{table}[t!]
\centering
\small
\sisetup{
round-mode = places,
round-precision = 0,
}%
\renewrobustcmd{\bfseries}{\fontseries{b}\selectfont}
\begin{tabular}{
@{} 
l
r
S[table-format=2.0] 
S[table-format=2.0]
S[table-format=2.0]
S[table-format=2.0]
@{} 
}
\toprule
\multirow{2}{*}{\bf Profession} & \multirow{2}{*}{\bf  Total} & \multicolumn{2}{c}{\bf Male} & \multicolumn{2}{c}{\bf Female} \\
  \cmidrule{3-6}
&       & {\bf  \$} & {\bf \st{\$}}   & {\bf \$} & {\bf  \st{\$}} \\
\midrule
        professor &  21715 &    46.1708 &   09.1964 &   37.3613 &   07.2715 \\
        physician &   7581 &    42.4482 &   08.4422 &   41.1159 &   07.9937 \\
         attorney &   6011 &    51.1728 &   09.9152 &   32.6901 &   06.2219 \\
     photographer &   4398 &    53.0923 &   11.0505 &   30.2638 &   05.5935 \\
       journalist &   3676 &    40.7236 &   09.3308 &   41.3765 &   08.5691 \\
            nurse &   3510 &    07.5499 &   01.1396 &   76.3818 &   14.9288 \\
     psychologist &   3280 &    30.7012 &   06.4939 &   52.2561 &   10.5488 \\
          teacher &   2946 &    35.1324 &   06.1439 &   49.2193 &   09.5044 \\
          dentist &   2682 &    52.0880 &   11.3348 &   30.2759 &   06.3013 \\
          surgeon &   2465 &    72.6572 &   12.3732 &   12.5761 &   02.3935 \\
        architect &   1891 &    64.1460 &   11.6341 &   20.7827 &   03.4373 \\
          painter &   1408 &    47.3011 &   08.8778 &   36.2926 &   07.5284 \\
            model &   1362 &    14.9046 &   02.4963 &   69.6035 &   12.9956 \\
             poet &   1295 &    45.9459 &   07.3359 &   38.5328 &   08.1853 \\
software engineer &   1289 &    69.7440 &   13.7316 &   14.0419 &   02.4825 \\
        filmmaker &   1225 &    55.5918 &   09.6327 &   28.8980 &   05.8776 \\
         composer &   1045 &    70.4306 &   14.1627 &   13.6842 &   01.7225 \\
       accountant &   1012 &    55.3360 &   09.4862 &   28.8538 &   06.3241 \\
        dietitian &    730 &    05.0685 &   01.2329 &   81.6438 &   12.0548 \\
         comedian &    499 &    69.3387 &   09.0180 &   18.6373 &   03.0060 \\
     chiropractor &    474 &    61.8143 &   14.3460 &   20.6751 &   03.1646 \\
           pastor &    453 &    59.3819 &   14.5695 &   22.5166 &   03.5320 \\
        paralegal &    330 &    12.4242 &   02.7273 &   70.0000 &   14.8485 \\
     yoga teacher &    305 &    13.4426 &   02.9508 &   71.4754 &   12.1311 \\
interior designer &    267 &    16.4794 &   04.1199 &   67.0412 &   12.3596 \\
 personal trainer &    264 &    41.2879 &   09.8485 &   42.0455 &   06.8182 \\
               DJ &    244 &    70.9016 &   15.5738 &   11.0656 &   02.4590 \\
           rapper &    221 &    74.6606 &   15.3846 &   09.0498 &   00.9050 \\
\midrule
\textbf{Total} & 72578 & 08.863567 & 45.079776 & 07.463694 & 38.592962\\
\bottomrule
\end{tabular} 
\caption{Training set distribution of the \Bios dataset. For each profession, the table shows the number of individuals and the breakdown across demographics as a percentage. \$ and \st{\$} denote the economic status (high vs.\ low, respectively). 
}
\label{table:bios_distribution}
\end{table}

\subsection{Dataset: \Bios}
All experiments are based on a biography classification dataset~\citep{de2019bias, ravfogel-etal-2020-null}, where biographies were scraped from the web, and annotated for the protected attribute of binary gender and target label of 28 profession classes.

Besides the binary gender attribute, we additionally consider economic status as a second protected attribute. \citet{subramanian-etal-2021-evaluating} semi-automatically labeled economic status based on the individual's home country (wealthy vs.\ rest of world), as geotagged from the first sentence of the biography.
For bias evaluation and mitigation, we consider the intersectional groups, i.e., the Cartesian product of the two protected attributes, leading to 4 intersectional classes: female--wealthy, female--rest, male--wealthy, and male--rest.

Since the data is not directly available, in order to construct the dataset, we use the scraping scripts of \citet{ravfogel-etal-2020-null}, leading to a dataset with 396k biographies.\footnote{There are slight discrepancies in the dataset composition due to data attrition: the original dataset~\citep{de2019bias} had 399k instances, while 393k were collected by~\citet{ravfogel-etal-2020-null}.}
Following~\citet{ravfogel-etal-2020-null}, we randomly split the dataset into train (65\%), dev (10\%), and test (25\%).

The augmentation for economic attributes follows previous work~\citep{subramanian-etal-2021-evaluating}, which results in approximately $30\%$ of instances that are labelled with both protected attributes.

\begin{figure*}[t]
    \centering
    \includegraphics[width=\linewidth]{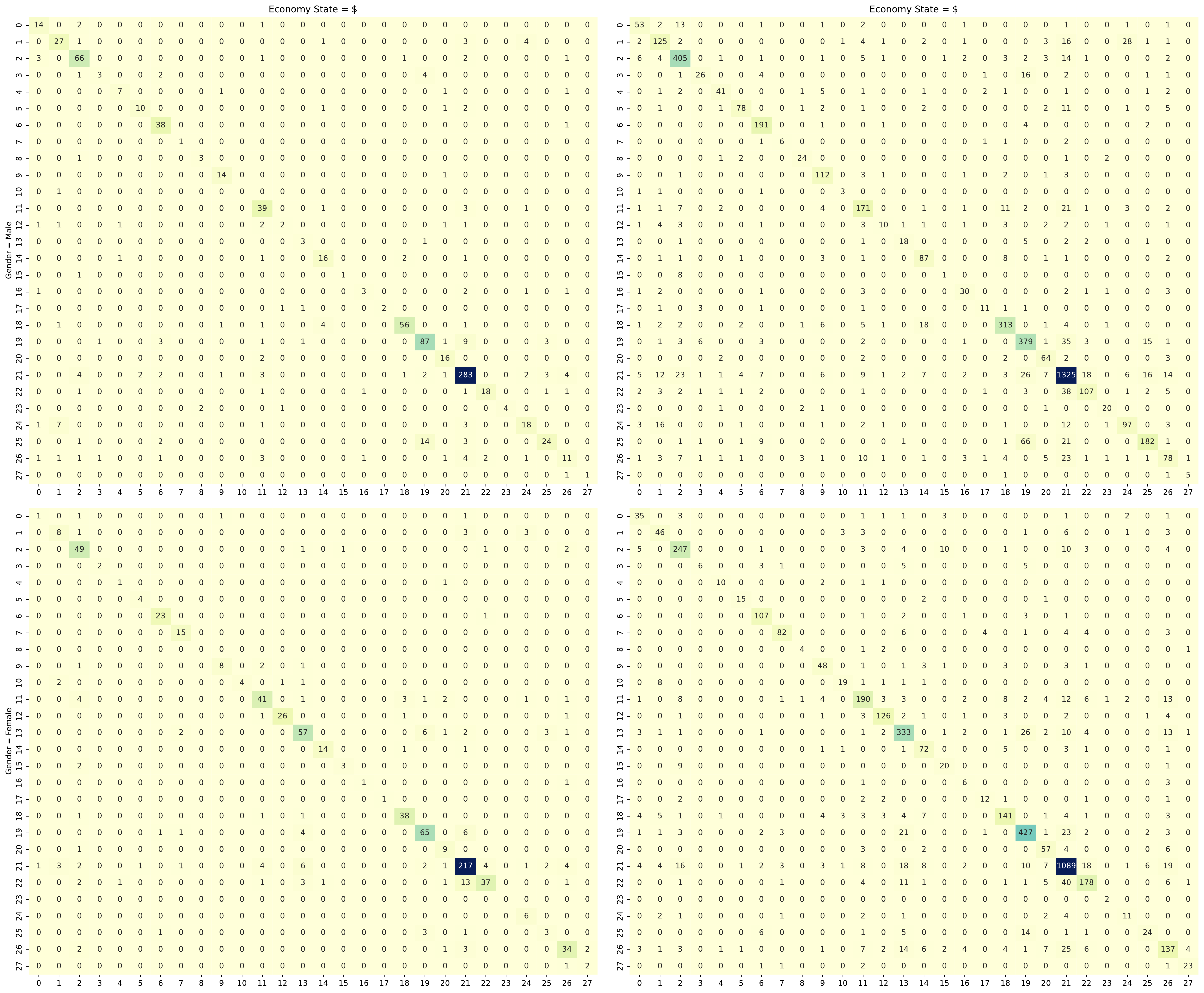}
    \caption{Confusion matrices of subgroups. Following Figure~\ref{fig:annotated_tpr_row_metrics_v}, confusion matrices are based on the predictions of Model 1, and class labels 0 to 27 are in the same order as the 28 professions in Figure~\ref{fig:annotated_tpr_row_metrics_v}.}
    \label{fig:subset_confusion_matrics}
\end{figure*}

\Cref{table:bios_distribution} shows the target label distribution and protected attribute distribution.

\subsection{Experimental Details}
This work focuses on evaluation and model comparison in the fairness literature. 
Instead of training models from scratch, we use existing checkpoints from previous work~\citep{han2022fairlib}, which are publicly available online.\footnote{Bios\_both at \url{https://github.com/HanXudong/fairlib/tree/main/analysis/results}}
Please refer to the original work~\citep{han2022fairlib} for experimental details.

\subsection{Subset Confusion Matrices}
\label{sec:Subset_Confusion_Matrices}

Figure~\ref{fig:subset_confusion_matrics} presents the confusion matrices of all 4 subgroups.
For each confusion matrix, the $i$-th row and $j$-th column entry indicates the number of samples which have the true label of the $i$-th class and predicted label of the $j$-th class.
Since the distributions of classes within each group can be highly imbalanced, without further normalization and aggregation, it is difficult to draw any conclusion by just observing the number of samples in each cell.

\subsection{Fairness Reproducibility}
So far, we have listed critical factors underlying the choice of fairness metric, and provided recommendations for metric selection.
However, we acknowledge that, in actual applications, the selection should be made in a domain-specific manner in close consultation with stakeholders or policymakers.
In practice, countless types of fairness evaluation metrics could be derived from different combinations of aggregation methods. 

Instead of reporting all possible fairness metrics, we suggest providing a set of confusion matrices for classification tasks, as it can form the basis of calculating a large number of metrics, including PPR, TPR, TNR, accuracy, and F-measure.
The other key advantage of reporting confusion matrices is that the number of reported values is generally much smaller than the model or dataset size.
Given a $\ermC$-class classification dataset with $\ermG$ distinct protected groups, the combined size  of the confusion matrices is  $\ermG \times \ermC^{2}$ (one confusion matrix per group). 
Taking the \Bios dataset as an example, the sizes of the confusion matrices, test dataset, and model parameters (for a BERT-base classifier~\citep{devlin2019bert}) are approximately $3 \times 10^{3}$, $4 \times 10^{4}$, and $1 \times 10^{8}$, respectively.

\begin{table*}[!ht]
\renewrobustcmd{\bfseries}{\fontseries{b}\selectfont}
\centering
\small
\begin{adjustbox}{max width=\linewidth}
\sisetup{
round-mode = places,
}%
\begin{tabular}{
ll
S[table-format=3.1, round-precision = 1]@{\,\( \pm \)\,}S[table-format=2.1, round-precision = 1,table-number-alignment = left]
S[table-format=3.1, round-precision = 1]@{\,\( \pm \)\,}S[table-format=2.1, round-precision = 1,table-number-alignment = left]
S[table-format=2.1, round-precision = 1]
S[table-format=3.1, round-precision = 1]@{\,\( \pm \)\,}S[table-format=2.1, round-precision = 1,table-number-alignment = left]
S[table-format=3.1, round-precision = 1]@{\,\( \pm \)\,}S[table-format=2.1, round-precision = 1,table-number-alignment = left]
S[table-format=2.1, round-precision = 1]
}

\toprule
  &  & \multicolumn{5}{c}{Test Set} & \multicolumn{5}{c}{Development Set} \\
\cmidrule(lr){3-7}\cmidrule(lr){8-12}

\bf Selection & \bf Method  & \multicolumn{2}{c}{\bf Performance} & \multicolumn{2}{c}{\bf Fairness} & \multicolumn{1}{c}{\bf DTO} & \multicolumn{2}{c}{\bf Performance }  & \multicolumn{2}{c}{\bf Fairness} & \multicolumn{1}{c}{\bf DTO} \\ 
\midrule
\multirow{4}{*}{\bf \DTO} &  \INLP &              81.354350 &              0.000000 &           62.442580 &           0.000000 & 41.931134    &     80.339540 &     0.000000 &      54.371000 &     0.000000 &   49.684398 \\
                          &   \Adv &              64.601630 &              4.508314 &           83.683840 &           1.123115 & 38.977707    &     63.862724 &     4.375978 &      79.931703 &     3.020925 &   41.335690 \\
                          &  \DAdv &              68.089476 &              5.485370 &           79.471282 &           6.769568 & 37.943508    &     67.579409 &     5.458177 &      75.463176 &     6.670896 &   40.658953 \\
                          &  \AAdv &              69.734659 &              4.898247 &           78.820603 &           7.722887 & 36.939920    &     69.295363 &     4.743128 &      75.200484 &     5.438290 &   39.468857 \\
\midrule
\multirow{4}{*}{\bf P} & \INLP &              81.354350 &              0.000000 &           62.442580 &           0.000000 & 41.931134    &   80.339540 &             0.000000 &          54.371000 &          0.000000 & 49.684398 \\
                                 & \Adv  &              81.470683 &              0.209556 &           59.474805 &           1.739385 & 44.560375    &   80.730194 &             0.220989 &          55.199163 &          1.304193 & 48.769257 \\
                                 & \DAdv &              81.436214 &              0.271583 &           59.287611 &           1.705075 & 44.744975    &   80.627966 &             0.251841 &          54.852240 &          1.709069 & 49.128362 \\
                                 & \AAdv &              81.256687 &              0.268492 &           58.635181 &           1.968618 & 45.413214    &   80.556773 &             0.287144 &          54.406716 &          1.578163 & 49.565983 \\
\midrule
\multirow{4}{*}{\bf P@F$+$5\%}  & \INLP &              53.753905 &              0.000000 &           74.979811 &           0.000000 &  52.580521 &  53.833516 &             0.000000 &          74.423048 &          0.000000 & 52.778071 \\
                                    & \Adv &               71.531363 &              5.669668 &           67.315544 &           5.354399 &  43.344399 &  71.254107 &             5.824511 &          66.183372 &          4.202679 & 44.383451 \\
                                    & \DAdv &              73.435065 &              3.763752 &           68.755527 &           4.905491 &  41.011131 &  72.964586 &             3.681593 &          65.952841 &          3.685704 & 43.475541 \\
                                    & \AAdv &              70.357258 &              5.722893 &           73.948014 &           8.864983 &  39.463883 &  69.985396 &             5.786656 &          67.419084 &          6.593549 & 44.298900 \\
\midrule
\multirow{4}{*}{\bf P@F$+$10\%}  & \INLP &              53.753905 &              0.000000 &           74.979811 &           0.000000 &   52.580521 &    53.833516 &             0.000000 &          74.423048 &          0.000000 & 52.778071 \\
                                     & \Adv  &              68.815482 &              4.795215 &           72.103062 &           5.577901 &   41.841526 &    68.152610 &             4.833248 &          70.271687 &          6.521562 & 43.566373 \\
                                     & \DAdv &              69.660694 &              2.900897 &           73.240677 &           6.832685 &   40.454108 &    69.207740 &             2.725972 &          70.604944 &          3.952069 & 42.570325 \\
                                     & \AAdv &              70.006104 &              3.340678 &           75.032929 &           3.874593 &   39.025484 &    69.448704 &             3.357603 &          70.820000 &          2.139696 & 42.247533 \\
 \midrule
\multirow{4}{*}{\bf F}  & \INLP &              29.812215 &              0.000000 &           99.999689 &           0.000000 &    70.187785 & 29.864914 &             0.000000 &          86.636821 &          0.000000 & 71.396812 \\
                                & \Adv &              51.601020 &             16.540178 &           90.219145 &           9.262941 &    49.377388 & 51.219423 &            16.237612 &          80.958107 &          6.000118 & 52.365431 \\
                               & \DAdv &              61.776597 &              3.693665 &           88.646096 &           3.685025 &    39.874048 & 61.151880 &             3.520992 &          82.157103 &          2.510622 & 42.749800 \\
                               & \AAdv &              37.946932 &              9.084083 &           99.038855 &           1.181159 &    62.060511 & 37.597663 &             8.771552 &          86.482231 &          0.174662 & 63.849680 \\
\midrule
\multirow{4}{*}{\bf F@P$-$5\%}  &  \INLP &              81.354350 &              0.000000 &           62.442580 &           0.000000 &  41.931134 & 80.339540 &             0.000000 &          54.371000 &          0.000000 & 49.684398 \\
                                    &   \Adv &              79.061434 &              1.052957 &           64.488537 &           1.304952 &  41.224841 & 78.504929 &             0.873399 &          58.732196 &          2.424519 & 46.530309 \\
                                    &  \DAdv &              80.443072 &              0.522515 &           64.657701 &           1.406390 &  40.392468 & 79.879518 &             0.545801 &          57.437202 &          1.137515 & 47.078929 \\
                                    &  \AAdv &              79.908082 &              2.177868 &           61.115353 &           2.754305 &  43.768721 & 79.136546 &             2.201642 &          58.067513 &          2.817180 & 46.836067 \\
\midrule
\multirow{4}{*}{\bf F@P$-$10\%} &  \INLP &              81.354350 &              0.000000 &           62.442580 &           0.000000 & 41.931134 & 80.339540 &             0.000000 &          54.371000 &          0.000000 & 49.684398 \\
                                    &   \Adv &              79.061434 &              1.052957 &           64.488537 &           1.304952 & 41.224841 & 78.504929 &             0.873399 &          58.732196 &          2.424519 & 46.530309 \\
                                    &  \DAdv &              74.207748 &              3.152859 &           66.965168 &           2.367657 & 41.911101 & 73.705732 &             3.284853 &          64.236675 &          1.143709 & 44.389232 \\
                                    &  \AAdv &              74.859790 &              5.121396 &           65.366052 &           6.868328 & 42.796501 & 74.267981 &             5.307684 &          64.246350 &          5.409667 & 44.050656 \\
\bottomrule
\end{tabular} 
\end{adjustbox}
\caption{Evaluation results $\pm$ standard deviation ($\%$) of selected models over the \Bios dataset. \DTO scores are the distance from mean performance and fairness to (1,1) over the test set. 
}
\label{table:full_model_selection_results}
\end{table*}

\section{Full Results of Case Studies}
\label{sec:full_comparison}

Table~\ref{table:full_model_selection_results} shows the experimental results for both the test and development sets.

\section{AUC-PFC Extension}
\label{sec:AUC_extension}

\subsection{Weighted DTO}
On the one hand, as suggested by~\citet{marler2004survey}, if fairness and performance have different scales, the Euclidean distance is not a suitable mathematical representation of closeness, resulting in worse approximation of Pareto optimality and efficiency. Therefore, the scales of performance and fairness should be normalized.

\subsection{Selection of Utopia Point}

\begin{figure}[t]
    \centering
    \includegraphics[width=\linewidth]{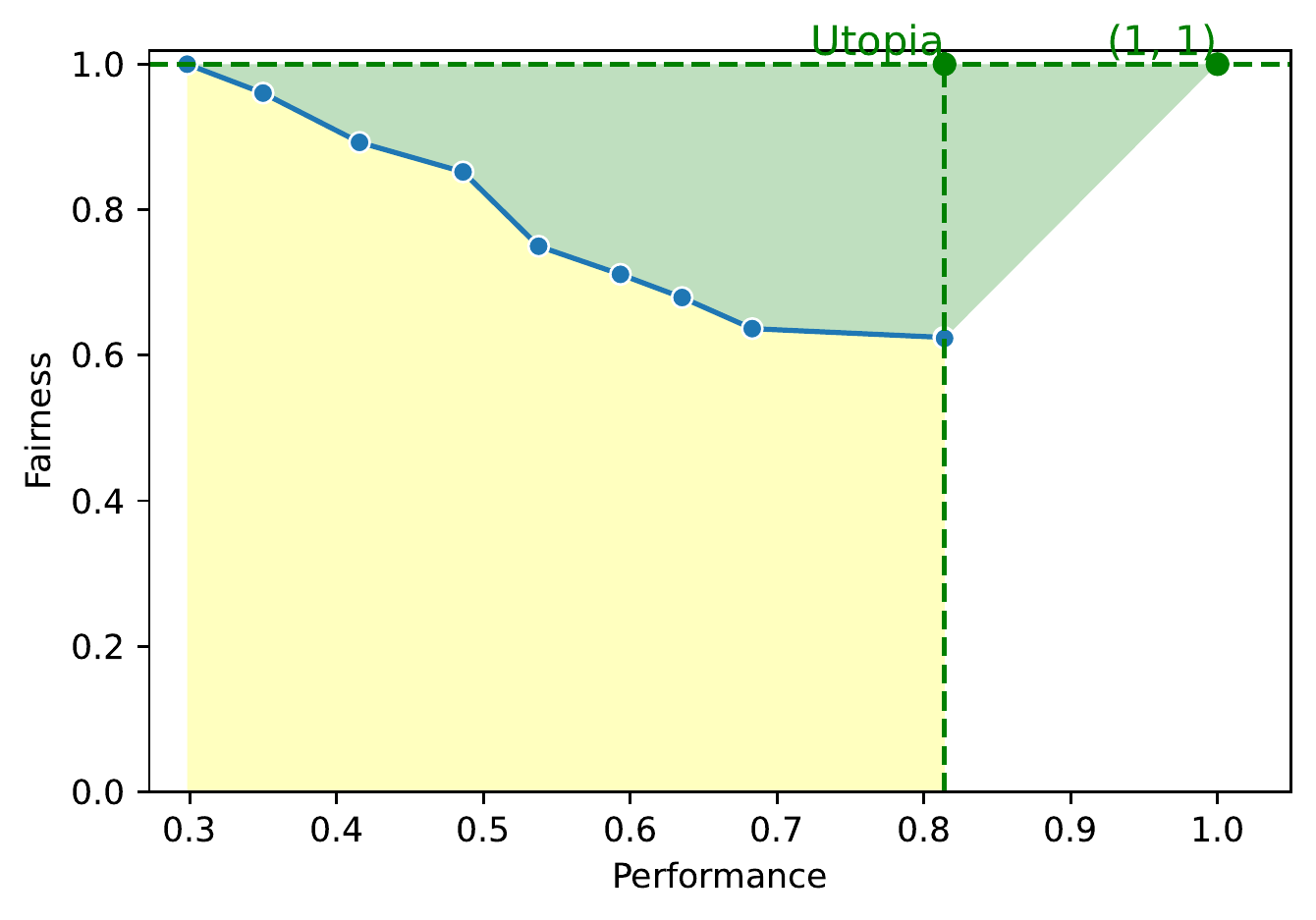}
    \caption{Integration of DTO with respect to an arbitrary ideal point, which is (1, 1) in this example.}
    \label{fig:Bios_both_INLP_AUC_extra}
\end{figure}

Typically, most debiasing methods will share the same maximum performance, which is the performance of the vanilla model (corresponding to a hyperparameter setting where the debiasing method does nothing.)
Accordingly this is a sensible choice for the performance of the Utopia point, as we have proposed for model selection.
In terms of the calculation of areas of integration, moving the Utopia point to $(1,1)$ has little effect, simply adding a constant triangular region which is identical for all methods, and thus irrelevant for model comparison.
As such, it makes no difference whether we use 1 or the maximum-achieved model performance when comparing models based on AUC-PFC.

\begin{figure}[t]
    \centering
    \includegraphics[width=\linewidth]{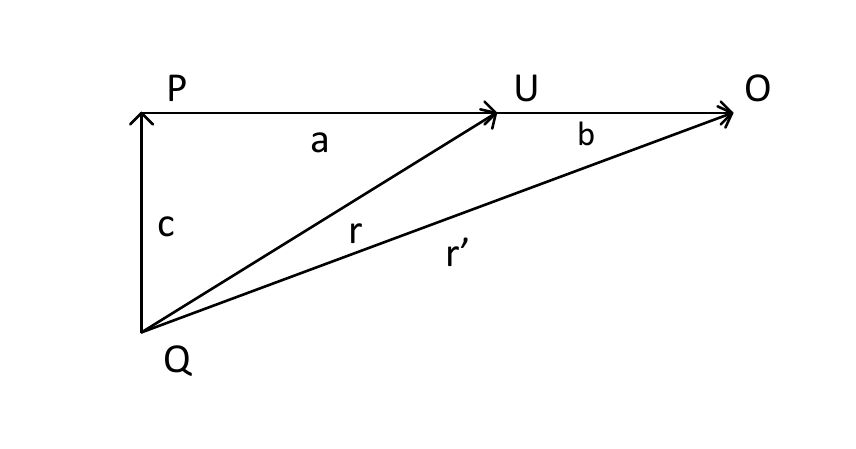}
    \caption{The influence of using different optimal points. Uppercase ($Q$, $P$, $U$, and $O$) and lowercase ($c$, $a$, $b$, $r$, and $r'$) characters represent points and Euclidean distances between points.}
    \label{fig:dto_ideal_point}
\end{figure}

\paragraph{Distance to Arbitrary Ideal Point:}
Compared to the default value of \DTO, moving the utopia point to the right (e.g., the $(1, 1)$ point) prioritizes methods with higher performance.

As shown in Figure~\ref{fig:dto_ideal_point}, without loss of generality, let 
\begin{itemize}
    \item $Q = (0,0)$ denote the candidate point;
    \item $U = (c,a)$ denote the Utopia point, where $c$ is the fairness distance from $Q$ to the maximum fairness (which is $1$), and $a$ is the performance distance from $Q$ to the maximum performance (which is $0.82$ is Figure~\ref{fig:dto_ideal_point}); and
    \item $O = (c,a+b)$ denote the arbitrary model, where $b > 0$, e.g., $b= (1-0.82)$ for the running example. 
\end{itemize}

Before discussing the influence of the optimum point selection, recall that the magnitude of vector sum, $|\vec{v}| = |\vec{v}_{1}+\vec{v}_{2}|$ is:
$$|\vec{v}| = \sqrt{ |\vec{v}_{1}|^{2} + |\vec{v}_{2}|^{2} + 2|\vec{v}_{1}||\vec{v}_{2}|\cos{\alpha} },$$
where $\alpha$ is the angle between $\vec{v}_{1}$ and $\vec{v}_{2}$.

Let $QU$ denote the vector from candidate model $Q$ to the Utopia point $U$, the DTO based on the Utopia point is the $r =\sqrt{a^2+c^2}$. 

When calculating \DTO based on the arbitrary optimum point $O$, 
$r' = |QU+UO|$, which can be shown as:
$$r' = \sqrt{r^{2} + b^2 + 2rb\cos{\alpha'}},$$
where $\alpha'$ is the angle between $QU$ and $UO$, and is equivalent to $\angle{PUQ}$. Furthermore, as discussed in \Cref{sec:AUC-PFC}, given a trade-off curve, the DTO is a function of $\angle{PUQ}$, i.e., the green shaded area is $\int_{0}^{\pi/2}\DTO(\angle{PUQ})d\angle{PUQ}$.

\begin{lemma}
Let $Q_{1}$ and $Q_{2}$ be two models with the same \DTO score ($r_{1}=r_{2}$), $r'_{1}$ and $r'_{2}$ be the \DTO to the new Utopia point $O$. If the performance of $Q_{1}$ is worse than $Q_{2}$, then $r'_{1} > r'_{2}$.

\end{lemma}

\begin{proof}
Assuming that $r'_{1}>r'_{2}$, 
{ \small
\begin{align}
\begin{split}
    |Q_{1}U+UO| & > |Q_{2}U+UO| \\
    |Q_{1}U+UO|^2 & > |Q_{2}U+UO|^2 \\
    2r_{1}b\cos{\angle{PUQ_{1}}} & > 2r_{2}b\cos{\angle{PUQ_{2}}} 
\end{split}
\end{align}
}
Since $\angle{PUQ} \leq \pi/2, \forall{Q}$, and $r_1 = r_2$, 
$$a_1 = r_{1}\cos{\angle{PUQ_{1}}} > a_2 = r_{2}\cos{\angle{PUQ_{2}}},$$
where $a_1$ and $a_2$ are the performance distances from $Q_1$ and $Q_2$ to the maximum performance, respectively. 

\end{proof}

\end{document}